\documentclass[letterpaper, 10 pt, conference]{ieeeconf}
\IEEEoverridecommandlockouts

\let\proof\relax
\let\endproof\relax
\usepackage{amsthm}
\usepackage[utf8]{inputenc}
\usepackage{amsmath,amssymb}
\usepackage{booktabs}
\usepackage{graphicx}
\usepackage[font=small]{subcaption}
\usepackage[dvipsnames]{xcolor}
\usepackage{numprint}
\usepackage{comment}
\usepackage[linesnumbered,ruled,noend]{algorithm2e}
\usepackage{eso-pic}

\usepackage{adjustbox}
\usepackage{tikz}
\usetikzlibrary{automata,positioning,shapes,arrows}

\usepackage{tabularx,siunitx}
\usepackage{cite}
\usepackage{hyperref}
\usepackage{multirow}

\newtheorem{theorem}{Theorem}
\newtheorem{problem}{Problem}

\newtheorem{assumption}{Assumption}

\newtheorem{lemma}{Lemma}
\newtheorem{definition}{Definition}

\newcommand{\reals}{\mathbb{R}}
\newcommand{\N}{\mathcal{N}}
\newcommand{\B}{\mathcal{B}}
\newcommand{\xhat}{\hat{x}}
\newcommand{\cov}{\Sigma_{k}}
\newcommand{\checkx}{\check{x}}
\newcommand{\checku}{\check{u}}
\newcommand{\checkX}{\check{X}}
\newcommand{\checkU}{\check{U}}

\newcommand{\prop}{\pi}
\newcommand{\Prop}{\Pi}

\newcommand{\pb}{\mathbf{b}}
\newcommand{\automaton}{\mathcal{M}}

\newcommand{\tempop}[1]{\mathcal{#1}}
\newcommand{\until}{\,\mathcal{U}}
\newcommand{\word}{\mathbf{\sigma}}
\newcommand{\eventually}{\lozenge}
\newcommand{\globally}{\square}
\newcommand{\sbd}{\text{SiMBA}\xspace}
\newcommand{\sbds}{\text{SiMBAs}\xspace}
\newcommand{\taskplan}{\mathbf{T}}

\title{\LARGE \bf Planning with SiMBA: Motion Planning under Uncertainty for Temporal Goals using Simplified Belief Guides}

\author{Qi Heng Ho, Zachary N. Sunberg, and Morteza Lahijanian%
\thanks{Authors are with the department of Aerospace Engineering Sciences at the University of Colorado Boulder, CO, USA
        {\tt\small \{\textit{firstname}.\textit{lastname}\}@colorado.edu}}
}
        
\begin{document}
\AddToShipoutPictureBG*{%
  \AtPageUpperLeft{%
    \hspace{16.5cm}%
    \raisebox{-1.5cm}{%
      \makebox[0pt][r]{To Appear in the IEEE Int'l. Conference on Robotics and Automation (ICRA), 2023.}}}}

\maketitle
\begin{abstract}
    This paper presents a new multi-layered algorithm for motion planning under motion and sensing uncertainties for Linear Temporal Logic specifications. We propose a technique to guide a sampling-based search tree in the combined task and belief space using trajectories from a simplified model of the system, to make the problem computationally tractable. Our method eliminates the need to construct fine and accurate finite abstractions. We prove correctness and probabilistic completeness of our algorithm, and illustrate the benefits of our approach on several case studies. Our results show that guidance with a simplified belief space model allows for significant speed-up in planning for complex specifications.
\end{abstract}
\section{Introduction}


As robots become more advanced, the expectation for them to perform tasks with higher complexities increases. 
It is thus an essential challenge to enable efficient planning for complex requirements. In addition, real-world robots must be able to reason about both motion and sensor uncertainty while executing such tasks. For instance, an autonomous underwater rover often has noisy motion due to water currents and gets accurate GPS measurements only at the surface; under the surface sensor readings are highly noisy. The combination of accounting for these uncertainties with the need to provide guarantees for task completion makes the planning problem very difficult. This paper focuses on this challenge and aims to develop an efficient framework for planning under uncertainty with complex specifications.

Linear Temporal Logic (LTL) is a principled formalism for expressing complex temporal tasks for robotic systems \cite{ltlf, Hadas2018Review, Fainekos2005}. For instance, a robot tasked with ``visit A and B in any order, and then go to C while avoiding D" can be expressed precisely as an LTL formula. The primary existing LTL motion planning frameworks for continuous state and action spaces are designed for deterministic systems \cite{Bhatia2010, Maly2013, Luo2021, Plaku2016}.

To address LTL synthesis problems for systems with uncertainty, works \cite{Wells2020, Ding2012, Luna:WAFR:2015} focus solely on motion uncertainty.  They first abstract the evolution of the robot in the environment into a finite Markov Decision Process (MDP), and then synthesize a policy that maximizes the probability of successfully completing the LTL specification. For observation uncertainty, LTL synthesis on Partially Observable MDPs (POMDPs) have recently been proposed for moderately-sized discrete space problems \cite{Bouton2020}, or through finite state abstraction of continuous spaces \cite{haesaert2018}. However, a common limitation of these works is the need for a finite abstraction from continuous states and actions to discrete ones. There are currently no methods that can give practically useful guarantees for general continuous-space POMDPs.



\begin{figure}[t]
    \centering
    \includegraphics[width=0.9\linewidth,  trim={2.5cm 4cm 3.5cm 4cm},clip]{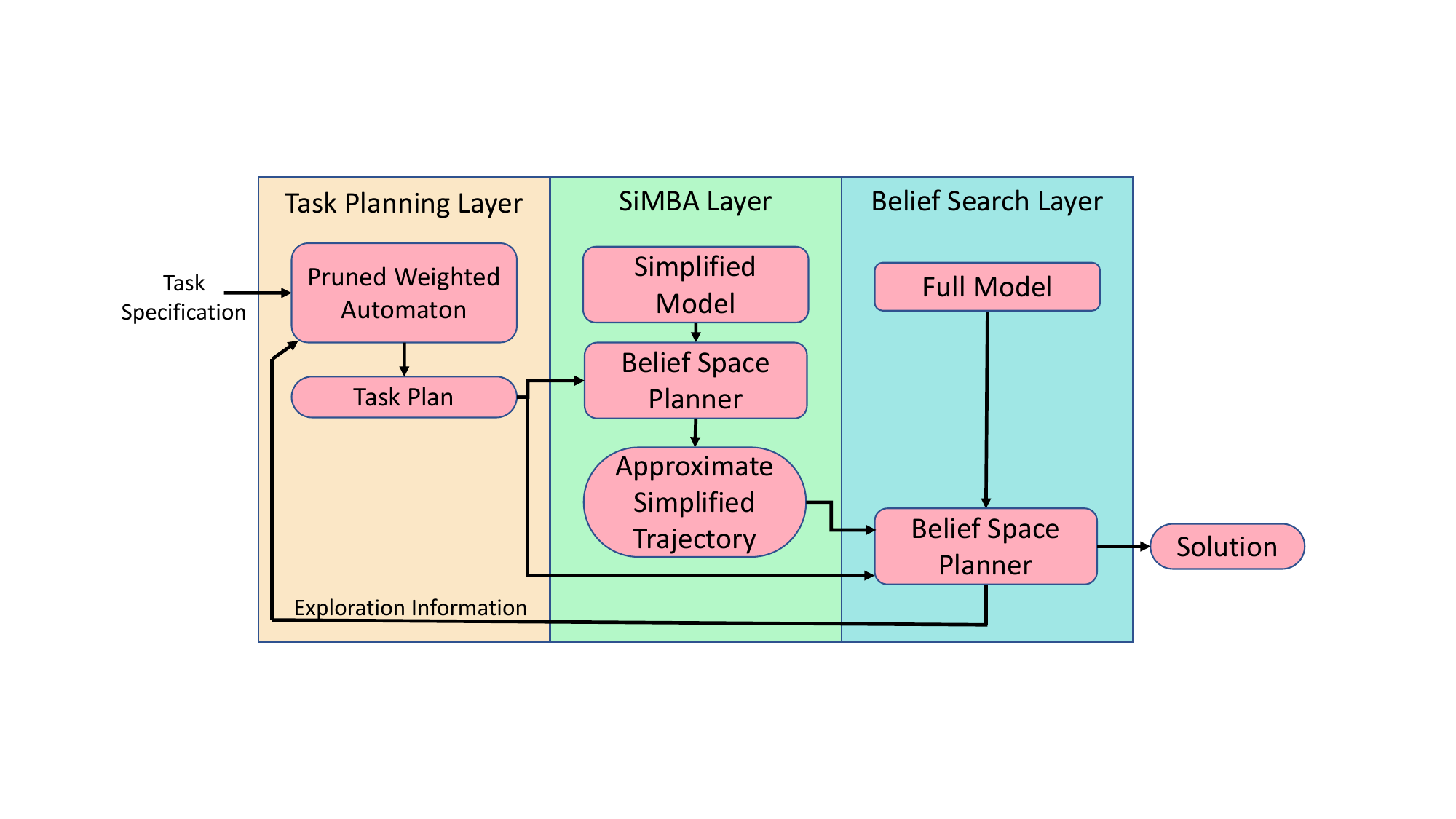}
    \caption{Planning with SiMBA Framework}
    \label{fig:overview}
    \vspace{-6mm}
\end{figure}

For continuous spaces, recent works in motion planning under uncertainty combine the power of sampling-based methods \cite{Kingston2018} with point-wise chance-constraints on the probability of safety \cite{cc-rrt, blackmore2011, cc-rrt*, firm, bry2011rrbt, Ho2022, Zheng2022}. This is a relaxation of maximizing success probability, but it allows for fast and scalable motion planning with robustness guarantees. Most works focus on motion uncertainty \cite{blackmore2011, pairet2021, cc-rrt}, but recent works \cite{bry2011rrbt, Ho2022, firm, Zheng2022} introduce extensions to account for sensor uncertainty. While they provide efficient planning under uncertainty, they are limited to simple task of A-to-B planning. Recent work \cite{Oh2021} extends chance constrained planning to LTL tasks and minimizes the probability of task failure, but their approach is limited to systems under motion uncertainty.

A main challenge in dealing with both uncertainty and complex specifications is the need to simultaneously account for both the belief and task spaces. Doing so results in extremely large search spaces. Principled guidance of a low-level motion tree has been shown to improve tractability in the deterministic setting for both single and multi-temporal goals planning, using layers of planning \cite{Plaku2010syclops, Maly2013, Ho2022CDC, Bhatia2010} or heuristic guidance \cite{Luo2021}. For stochastic systems, \cite{Tajvar2020} proposes a heuristic to guide tree extension, but their framework is only designed for dynamics noise. Extending these techniques to settings with both motion and measurement uncertainty is non-trivial because of the difficulty in constructing a guidance mechanism that captures belief dynamics well.

In this paper, we present a multi-layered framework to synthesize motion plans for LTL tasks for systems under both motion and measurement uncertainty. The LTL tasks are defined in the belief space with a user-defined robustness requirement. We show that our framework provides guarantees on task completion and probabilistic completeness. 

Our approach has two novel characteristics: First, we introduce a method to plan over an automaton that represents the LTL task directly that avoids the need for fine abstraction, by pruning infeasible edges. Second, in order to make the problem computationally tractable, we propose to use a simplified model that accounts for both motion and measurement uncertainty. We refer to this notion as Simplified Models with Belief Approximation (\sbd), which we use to rapidly find a continuous path that heuristically guides the search for a satisfying motion plan for the full system. Our evaluation shows that the addition of \sbd for guidance improves solution time significantly. To the best of our knowledge, this is the first work that uses trajectories from a simplified model in the belief space to bias sampling-based search for stochastic systems under complex tasks.

In summary, the contributions of this paper are five-fold: (i) a framework for planning under LTL specifications for continuous systems under motion and measurement uncertainty with (ii) guarantees on task completion and probabilistic completeness; (iii) planning over a pruned automaton that alleviates the state explosion problem in abstraction-based methods; (iv) improving efficiency by using a simplified model in the belief space to provide continuous trajectory guides for belief planning, and (5) a series of illustrative case studies and benchmarks.





\section{Problem Formulation}
\label{sec:problemstatement}

Consider a robot with both motion and sensing uncertainty tasked with a complex navigation task in a bounded workspace $W \subset \reals^d, d \in \{2, 3\}$. We are interested in computing a motion plan for the robot to achieve its task with guarantees. Below, we formalize this problem.


The robot has linear or linearizable motion and measurement models given by:
{\small
\begin{equation}
    \label{eq:system}
    \begin{split}
        &x_{k+1} = Ax_{k} + Bu_{k} + \omega_k, \;\;\;\,\omega_k \sim \N(0, \mathcal{Q}),\\
        &z_k = Cx_k + Du_{k} + \nu_k(x_k),  \;\nu_k(x_k) \sim \N(0, \mathcal{R}(x_k)),
    \end{split}
\end{equation}
}%
where $x_k \in X \subset \reals^n$ is the state, $u_k \in U \subset \reals^m$ is the input, and $z_k \in \reals^p$ is the measurement with their corresponding matrices $A \in \reals^{n\times n}$, $B \in \reals^{n\times m}$, and $C \in \reals^{n\times p}$. Noise terms $\omega_k$ and $\nu_k$ are i.i.d. random variables with zero-mean Gaussian distributions and covariance matrices $\mathcal{Q}$ and $\mathcal{R}(x_k)$, respectively. Note that covariance $\mathcal{R}(x_k)$ depends on state $x_k$, i.e., it is state-varying measurement noise. This describes problems common in robotics where measurements may be available only in parts of the environment or measurement quality depends on the location (state) of the robot.

The evolution of the robotic system in \eqref{eq:system} can be described by a discrete-time Gauss-Markov process \cite{gaussmarkovprocess}. The robot state $x_k$ at each time step is a random variable with a Gaussian distribution, i.e., $x_k \sim b_k = \N(\xhat_k, \cov)$, where $b_k \in \B$ is referred to as the \textit{belief} of the robot state, $\B$ is the \textit{belief space}, and $\xhat_k$ and $\cov$ are the state mean and covariance matrix, respectively. A \textit{belief trajectory} $\pb = b_0b_1b_2\cdots b_f$ is a sequence of beliefs.


\subsection{Linear Temporal Logic over Finite Traces for Belief Tasks} 

The robot is tasked with a temporal specification over the belief space, similar to \cite{Sadigh2016RSS, haesaert2018, Silva2021, CCTL2018, perceptionuncertaintyLTL}.
Let $\mathcal{P} = \{p_1, \cdots, p_n\}$ be a set of convex
polytopic regions in the state space $X$, i.e., $p_i \subset X$ for all $1 \leq i \leq n$.
Further, let $\alpha_i \in [0,1]$ be a probability threshold (chance constraint) on $p_i$.  Then, we define atomic proposition $\pi_i^{\alpha_i}$ associated with region $p_i$ as follows: $\pi_i^{\alpha_i}$ is true for belief $b_k$ iff $P(x_k \in p_i) > 1 - \alpha_i$. This definition states that a proposition is true if the probability that a belief state is in a region $p_i$ is at least $1 - \alpha_i$.
Hence, the set of atomic propositions is $\Pi = \{\pi^{\alpha_i}_i \mid i = 1, \cdots, n\}$.  

We define a labeling function $L : \B \rightarrow 2^\Pi$ that maps a belief $b \in \B$ to the set atomic propositions in $\Pi$ that are true for $b$. 
Then, a belief trajectory $\pb = b_0\cdots b_f$ generates an observation trace (word) $\word = \word_0 \cdots \word_f$ via the labeling function such that $\word_i = L(b_i)$ for all $0 \leq i \leq f$.  Each $\word_i \in 2^\Pi$ is called a symbol.

\begin{definition}[LTLf Syntax \cite{ltlf}]
 Let $\Pi$ be a set of atomic propositions. Then, an LTLf formula over $\Pi$ is recursively defined by
    \begin{align*}
        \phi := \pi^\alpha \mid \neg \phi \mid \phi \vee \phi \mid \tempop{X} \phi \mid \phi \until \phi
    \end{align*}
where $\pi^\alpha \in \Pi$ is an atomic proposition, $\neg$ (``negation'') and $\vee$ (``or'') are Boolean operators, and  $\tempop{X}$ (``next'') and $\until$ (``until'') are the temporal operators.
\end{definition}

\noindent From this definition, one can derive other standard temporal operators, e.g., the $\eventually$ (``eventually") and $\globally$ (``globally") operators are given by $\eventually \phi \equiv \top \until \phi$ and $\globally \phi \equiv \neg \eventually \neg \phi$.

LTLf formulas are interpreted over finite words in $(2^{\Pi})^*$. The semantics of LTLf can be found in \cite{ltlf}. We say a finite trace $\sigma$ satisfies formula $\phi$, denoted by $\sigma \models \phi$, iff $\sigma, 0 \models \phi$. 
Additionally, we say that a belief trajectory satisfies a specification $\phi$, written $\mathbf{b} \models \phi$, if its observation trace (word) $\word = L(b_0)L(b_1)\cdots L(b_n)$ satisfies $\phi$, i.e., $\word \models \phi$.
Note that, for a small $\alpha$, $\pi_i^\alpha$ requires being in region $p_i$ with a high probability.  
The negation of it, $\neg \pi^{\alpha} = P(x_k \in p_i) \leq 1 - \alpha$, still allows being in region $p_i$ with a relatively high probability for $\alpha \ll  1$. In many applications, however, one may be interested in expressing a task for being in a region with high probability in one part of a formula and then avoiding the same region also with high probability in another part of the formula. Similar to \cite{Nathalie2019}, we can do this by capturing the more intuitive converse of $\pi^{\alpha}_{i}$ by adding new atomic propositions. This can be done by associating to each $p_{i}$ a new atomic proposition $\tilde{\pi}^{\alpha}_{i}$ that is true for belief $b_k$ iff $P(x_k \notin p_i) \geq 1 - \alpha$, such that $\tilde{\pi}^{\alpha}_{i}$ represents $\neg {\pi}^{1 - \alpha}_{i}$.

\subsection{Planning Problem}
\label{sec:belief_propagation_equations}


We seek a motion plan that satisfies a given LTLf formula $\phi$. 
Assuming that the robot is equipped with a feedback trajectory-following controller,
a motion plan can then be characterized as a sequence of control inputs coupled with its corresponding nominal state trajectory.
Let $\checkU^{0,t} = (\checku_0, \checku_1, \ldots, \checku_{t-1})$ be a sequence of control inputs. Then, given an 
initial belief $b_0$, a nominal trajectory
$\checkX^{x_0,x_t} = (\checkx_0, \checkx_1, \ldots, \checkx_t)$ is obtained by applying $\checkU^{0,t}$
to the nominal system dynamics $\checkx_{k+1} = A\checkx_{k} + B\checku_k$. This motion plan $(\checkU^{0,t}, \checkX^{x_0,x_t})$ is then executed online, given online state estimates $\xhat_k$, via the stabilizing, trajectory-following controller
\begin{align}
    \label{eq:feedback controller}
    u_k = \check{u}_{k-1} - K (\xhat_k - \checkx_k),
\end{align}
where $K$ is the feedback control gain. 




\begin{problem}
    \label{prob:1}
    Given System \eqref{eq:system}, a set of probabilistic atomic propositions $\Prop = \{\prop^{\alpha_1}_1, \ldots, \prop^{\alpha_n}_n\}$ defined over regions $\mathcal{P}$, LTLf formula $\phi$ defined over $\Prop$, and closed-loop control gain K, find a motion plan $(\checkU, \checkX)$ as a pair of sequence of nominal controls $\checkU = (\checku_0,\ldots,\checku_{k-1})$ for some $k\geq 1$ and its resulting nominal trajectory $\checkX = (\checkx_0,\ldots, \checkx_k)$ such that when executed via the controller in \eqref{eq:feedback controller} the resulting belief trajectory $\pb$ satisfies the LTLf formula, i.e., $\pb \models \phi.$
\end{problem}



The main challenge in this problem is that the search space is extremely large. This is because planning must be performed in the composition of the belief space with the task space (representation of the set of all satisfying words). For an $n$-dimensional state space, the Gaussian belief space adds $O(n^2)$ dimensions, and for a task formula of size $|\phi|$, the size of the task can be as large as $2^{2^{|\phi|}}$ \cite{ltlf}. Hence, to make the problem computationally tractable, the search must be performed with informed guidance. Existing methods for guiding search in the state space through, e.g., discrete abstractions or geometric planning, fail to capture belief information, leading to uninformative guides. 
To address these challenges, we aim to design a framework that reduces dependence on fine, discrete abstractions, while efficiently searching in the belief space. Specifically, we seek to generate informed guidance via simplified models.

\section{Belief Space Planning with Simplified Belief Approximations}
\label{sec:beliefplanning}

We present a modular framework to solve Problem~\ref{prob:1}. The framework is inspired by \cite{Plaku2010syclops, Maly2013, pairet2021} and consists of three main layers: task planning layer, \sbd guide layer, and belief search layer, as depicted in Fig. \ref{fig:overview}.

We first translate the LTLf formula $\phi$ into a minimal Deterministic Finite Automaton (DFA). Then, we prune unrealizable transitions of the DFA, and use the pruned DFA for task planning. At each iteration, we find the shortest path on the DFA that corresponds to a $\phi$-satisfying trace to use as a task plan. Next, at the \sbd guide layer, we use a simplified model of the system to rapidly find a path annotated with DFA states that obeys the task plan. We call this technique Simplified Models with Belief Approximation (\sbd). Finally, at the belief search layer, a belief tree searches for a satisfying trajectory for the full system. The growth of this tree is biased to search around the hybrid annotated path while ensuring the motion plan satisfies $\phi$.


\subsection{Belief Propagation and Labeling}
\label{sec:beliefpropagationandlabel}


Here, we establish the relationship between the evolution of the distribution over robot state, atomic predicates, and labels.
During offline planning, following \cite{bry2011rrbt, Ho2022}, given a stabilizing feedback controller in \eqref{eq:feedback controller} and online state estimate $\xhat_k$, 
we parameterize the belief at each time step based on the closed-loop dynamics as
{\small
\begin{align}
    \label{eq:belief}
    x_k \sim b_k = \N(\checkx_k, \cov^+ + \Lambda_k^+ ),
\end{align}
}%
where $\cov^+$ is the covariance that represents the online state estimation error, and $\Lambda_k^+$ is the covariance matrix of $\hat{x}_k$, which is a random variable when planning offline. It 
accounts for the uncertainty related to the measurements when executed online.
Analogous to the Kalman Filter, covariance matrices $\cov^+$ and $\Lambda_k^+$ can be computed recursively:
{\small
\begin{align}
    \label{eq:belief_propagation_equations}
    \cov^- &= A\Sigma_{k-1}^+A^T + Q,\\
    L_k &= \cov^-C^T(C\cov^-C^T + R(\xhat_k))^{-1},\\
    \cov^+ &= \cov^- - L_kC\cov^-,\\
    \Lambda_k^+ &= (A- BK)\Lambda_{k-1}^+(A-BK)^T + L_kC\cov^- \text{.}\label{eq:beliefpropagation}
\end{align}
}%
Given belief $b_k$, its label $L(b_k)$ is the set of atomic propositions in $\Pi$ that are true in $b_k$, i.e., $\pi_i^{\alpha_i} \in L(b_k) \iff P(x_k \in p_i) > 1 - \alpha,$ where
{\small
\begin{align}
    \textstyle
    \label{eq:belief integral}
    P(x_k \in p_i) = \int_{p_i} \N(s \mid \checkx_k, \cov^+ + \Lambda_k^+ ) ds.
\end{align}
}%
\noindent
Since our regions are convex polytopes that can be described as a conjunction of linear half-spaces, we can 
efficiently compute bounds for \eqref{eq:belief integral}, using methods such as those described in \cite{blackmore2011, cc-rrt, Liu2014, pairet2021}.

\subsection{Task Planning with Pruned Automaton}

At the task planning layer, an automaton is used to define a sequence of sub-tasks the robot has to complete.

\subsubsection{Automaton Construction and Pruning}

An LTLf formula $\phi$ can be translated into a minimized DFA 
that represents precisely the traces that satisfies $\phi$ \cite{ltlf}. 
A DFA is defined as a tuple $\automaton = (Q, q_0, S,  \delta, F)$, where $Q$ is a finite set of states, $q_0 \in Q$ is the initial states, $S= 2^\Pi$ is a set of input symbols, $\delta : Q \times S \rightarrow 2^Q$ is the transition function, $F \subseteq Q$ is the set of accepting states.

Instead of planning over a product automaton (Cartesian product of $\automaton$ with a finite abstraction of the robotic system) like in \cite{Oh2021, Maly2013, Bhatia2010, Kantaros2020, Ho2022CDC}, we propose to use the abstraction to prune impossible task sequences of $\automaton$. To do this, we construct an adjacency graph of the regions in $\mathcal{P}$ in order to prune impossible transitions in $\automaton$. 


The adjacency graph is constructed by computing adjacency and intersections between the regions in $\mathcal{P}$ as well as the remainder region $X \setminus \cup_{p \in \mathcal{P}}$. Each region is represented as a node in the adjacency graph, and edges of the graph represent regions which are geometrically adjacent to each other or intersecting with one another. 
This adjacency graph captures the possible transitions between regions in the continuous state space, and we use it to remove geometrically impossible letters in the alphabet, thereby pruning the edges in $\automaton$ that are not realizable in the state space. 



    This approach has two advantages. First, it removes the need for fine abstractions of the problem space and system, which improves scalability by alleviating the state explosion problem. Our approach generalizes the method proposed by \cite{Luo2021}, by providing a way to prune the automaton using information about $\mathcal{P}$. Second, our method gives the belief planners in the subsequent layers the freedom to explore the belief space without restrictions, which is important for belief space planning, especially when the measurement noise is state dependent as in our problem, i.e., $\mathcal{R}(x_k)$. 

Task planning is directly performed on $\automaton$ augmented with edge weights that represent estimates on motion tree feasibility. 
The edge weighting scheme accounts for the fact that some transitions on the adjacency graph
could be geometrically possible  (and thus exist in $\automaton$) but are not realizable by the system's belief dynamics.
These weights are continuously updated during the planning process.

\subsubsection{Task Planning}
In each iteration of task planning, 
we compute 
an accepting run (task plan) $\taskplan$ 
on the pruned and weighted $\automaton$.
A graph search algorithm, such as $A^*$, is used to find the shortest path to an accepting state in $F$ from the initial state $q_0$. Each task plan is a candidate sequence of automaton states that the algorithm uses to lead 
the search in the
subsequent layers. By finding a belief trajectory that follows $\taskplan$, the robot is guaranteed to satisfy 
$\phi$.


To define the feasibility edge weights of $\automaton$, we first assign a weight to each state, similar to \cite{Bhatia2010}.
For state $q \in Q$,
\begin{align}
    \textstyle
    \label{eq:automatonweights}
    w(q) = \frac{(cov(q) +1)}{DistFromAcc(q) \cdot (numsel(q) +1)^2},
\end{align}
where $cov(q)$ is the number of motion tree vertices associated with $q$, $numsel(q)$ is the number of times $q$ has been selected, and $DistFromAcc(q)$ is the shortest unweighted path from $q$ to an accepting state in $F$. $DistFromAcc(q)$ can be computed using an unweighted graph search algorithm such as Djikstra's Algorithm and can be computed once beforehand. For a minimized DFA, a path to an accepting state always exists from any state. Then, the edge weight between states $q,q' \in Q$ is $w(q,q') = (w(q)\cdot w(q'))^{-1}.$ This weighting scheme promotes search in unexplored areas of the task space and suppresses search in areas where attempts at finding a solution have repeatedly failed.
\subsection{Hybrid Tree Search: Gaussian Belief Trees with DFAs}
\label{sec:gbt}
For belief space planning with both the simplified model and full constrained model, we use Gaussian Belief Trees (GBT) in \cite{Ho2022} as our base belief space planner in a hybrid discrete-continuous tree search planner. The hybrid planner is similar to the frameworks for deterministic LTL planners in \cite{Maly2013,Bhatia2010, Ho2022CDC}, but it directly plans with automaton states instead of product automaton states. Also, it reasons about robot uncertainty while planning. For completeness of presentation, we provide a brief overview of this method.

Given a task plan $\taskplan$, we compute the set of DFA states $Q_T$ that exist in $\taskplan$ and contain at least one tree vertex at every iteration. The hybrid tree search expands in $Q_T$. In each iteration, a DFA state $q$ is sampled from $Q_T$. A tree vertex $\text{v}_s$ is sampled among vertices that have discrete components $q$, and one iteration of GBT is performed to obtain a new vertex $\text{v}_n$. A tree vertex is a tuple $\text{v} = (\check{x}, \Sigma, \Lambda, q)$. An iteration of GBT involves selecting a tree vertex $\text{v}_{sel}$, sampling a valid nominal control $\check{u} \in U$, and propagating the continuous belief components in $\text{v}_{sel}$ to obtain the new belief $(\check{x}_n, \Sigma_n, \Lambda_n)$, according to \eqref{eq:system} and Sec.~\ref{sec:belief_propagation_equations}. The new discrete state component $q_n$ of $\text{v}_n$ is obtained by propagating the automaton with the label of the new belief.



Vertex $\text{v}_n$ is valid and added to the tree if both its continuous and discrete state components are valid. The continuous component $\text{v}_n.b = (\text{v}_n.\check{x}, \text{v}_n.\Sigma, \text{v}_n.\Lambda)$ is valid if it obeys the constraints of \eqref{eq:system}. The discrete component $\text{v}.q$ is valid if it exists in the DFA. Finally, if $\text{v}_n.q$ is an accepting state, the motion plan $(\check{U}, \check{X})$ and corresponding tree vertices $\mathbf{\text{v}}$ ending with $\text{v}_n$ is returned as a solution.


\subsection{\sbd Layer for LTLf Guides}
\label{sec:simba}

Once a task plan $\taskplan$  is generated, we can directly use the hybrid tree search with GBT on System~\eqref{eq:system} to attempt to satisfy the task sequence constraints, as described above. However, a naive method of tree search is computationally inefficient in the absence of a way to guide the search to promising areas of the large search space. To address this issue, we 
take inspirations from \cite{pairet2021} and
utilize trajectory biasing, and seek to rapidly find a path that satisfies the task plan. We propose to do this by using a simplified motion model of the system and include approximate belief dynamics (\sbd). By simplifying the dynamics while still retaining some belief information, we are able to obtain a calculated trade-off between speed of finding a simplified solution path for guidance and informative guides in the belief space. Below, we formalize this idea for \sbd.

First, we create a lower dimensional state space $\tilde{X}$ which is a subspace of the state space $X$ based on the projection $\tilde{x} = Proj(x)$ 
where $Proj: X \rightarrow \tilde{X}$ maps $x \in X$ to $\tilde{x} \in \tilde{X}$. This simplified state space can be arbitrarily chosen, with the only condition being that the new subspace $\tilde{X}$ contains the subspace of all the polytopic regions $\mathcal{P}$, i.e., $p_i \subseteq \tilde{X}$ $\forall p_i \in \mathcal{P}$. If the polytopic regions require all dimensions of the state space, \sbd can still be used, albeit with $\tilde{x} = x$. 

Second, we design a simplified motion model and measurement model for $\tilde{x}$. These models can also be arbitrarily chosen, but a general rule is to have simpler models for $\tilde{x}$ than $x$, since the purpose of \sbd is to compute approximate solutions rapidly in order to guide the belief search layer. Some generally applicable examples of simplified models are those with simple first order dynamics and kinematic models. One could also use simplifications such as: (i) geometric planning that ignores uncertainty, (ii)  System~\eqref{eq:system} without the process and measurement noise terms, and (iii) System~\eqref{eq:system} without measurement noise. 



For many robotic LTL tasks, the regions of interest are defined in the workspace. Therefore, we seek to design a motion model that generates kinematic trajectories. To do this, we take advantage of the observation that geometric paths can be interpreted as kinematic trajectories with a fixed speed in each direction. Thus, 
we use the maximum speed $v_{max}$ of the robot. For the $i$th component of state $\tilde{x}$ denoted by $\tilde{x}_{k}^{(i)}$, the motion model becomes:
\begin{align}
    \textstyle
    \tilde{x}_{k+1}^{(i)} = \tilde{x}_{k}^{(i)} + v_{i},
\end{align}
such that $\sum v_{i}^2 = v_{max}^2$. As a simple approximation, we maintain the covariance propagation of the original system using \eqref{eq:belief_propagation_equations}-\eqref{eq:beliefpropagation}. We refer to this \sbd as Simple Belief Approximation-\sbd (SBA-\sbd).






\begin{definition}[Admissible \sbd]
    Consider System \eqref{eq:system} with $X$ and its subspace $\tilde{X}$ for a \sbd. The \sbd is admissible if, for every instance of Problem~\ref{prob:1} that admits a solution for System \eqref{eq:system}, a solution also exists for the \sbd in $\tilde{X}$.
\end{definition}



Since the purpose of \sbd is to find a coarse solution quickly to guide the belief search layer, it is desirable to use admissible \sbds. However, the algorithm can still be probabilistically complete with an inadmissible \sbd, as we discuss in Sec.~\ref{sec:analysis}.

In many robotics scenarios, low robot uncertainty is more desirable than high uncertainty. The intuition is that large uncertainty leads to higher risks (e.g., collision probability). However, mathematically, when bad regions (e.g., obstacles) are small, large uncertainty can reduce risk. To rule out such unintuitive cases, we make the following assumption, which we also use to build our \sbd methodology.

\begin{assumption}
    \label{as:covariance decrease phi better}
    If $b_a = \mathcal{N}(\check{x}, \Sigma_a)$, $b_b = \mathcal{N}(\check{x}, \Sigma_b)$, and $\Sigma_a < \Sigma_b$, then 
    \begin{itemize}
        \item for all $\alpha < 0.5$, $\pi^{\alpha} \in L(b_b) \implies \pi^{\alpha} \in L(b_a)$, and
        \item for all $\alpha > 0.5$, $\pi^{\alpha} \notin L(b_b) \implies \pi^{\alpha} \notin L(b_a)$
    \end{itemize}
\end{assumption}

Assumption \ref{as:covariance decrease phi better} states that increasing the covariance does not increase the probability of being in a region at a given state mean for low $\alpha$, which corresponds to `reach' regions. Additionally, increasing the covariance does not increase the probability of not being in a region at a given state mean for high $\alpha$, which corresponds to `avoid' regions. 

\begin{lemma}
    SBA-\sbd is admissible under Assumption \ref{as:covariance decrease phi better}.
\end{lemma}
\begin{proof}[Proof Sketch]
    Starting from the same $\tilde{x}_0$ and the same covariance, SBA-\sbd is able to under-approximate the uncertainty of the belief of the original system at any $\tilde{x}$, since there are less constraints on its dynamics and it uses the maximum speed of the original system. From Assumption~\ref{as:covariance decrease phi better}, a lower covariance at each state leads to better probability of satisfaction. There exists at least as many accepting belief trajectories for \sbd as for the full system.
\end{proof}
Finally, the goal is to find a feasible solution for the \sbd augmented version of Problem~\ref{prob:1} in each iteration of the planning framework. Using the hybrid search tree in
Sec.~\ref{sec:gbt}, we obtain a kinematic trajectory for the simplified space $\tilde{X}$. A \sbd path is then the hybrid path $\xi_{\tilde{X}} = (\mathbf{\tilde{x}}, \mathbf{q})$ which is a sequence of pairs of $\tilde{x}$ and $\automaton$ state $q$. 

We note that designing simplified models can generally be difficult, but our results show that using SBA-SiMBA significantly reduces planning time for various LTLf tasks.

\subsection{Belief Search Layer}
\label{sec:fullplanning}


The belief search layer plans for the original constrained System~\eqref{eq:system} using GBT to find a solution for Problem~\ref{prob:1}. To guide the search, we lift \sbd path $\xi_{\tilde{X}}$ from $\tilde{X}$ to $X$ and use biased sampling around the lifted trajectory set.



We propose to guide the belief tree search by using the \textit{sub-task relevant} segments of $\xi_{\tilde{X}}$.
Let $\text{v}$ be an existing tree node in the belief search layer with discrete component $\text{v}.q_i$, and let $q_{i+1}$ be the successor of $q_i$ in task plan $\mathbf{T}$.
Then, we use the segment of $\xi_{\tilde{X}}$ with discrete components $q_i$ and those that lead to a transition from $q_i$ to $q_{i+1}$, 
i.e., the segment is of the form $\xi^{q_i}_{\tilde{X}} = (\tilde{x}_k, q_i)...(\tilde{x}_{k+m}, q_i)(\tilde{x}_{k+m+1}, q_{i+1})$, where $k$ is the first time step $\xi_{\tilde{X}}$ encounters $q_i$ and $m \geq 0$.

The sampling bias technique is as follows. Given a starting radius $d$ which is incrementally increased, samples are chosen within $d$ radius of $\xi^{q_i}_{\tilde{X}}$ with probability $pr$, and uniformly in the belief space with probability $1-pr$. Biasing sampling in this way allows the belief search tree to bias growth towards promising parts of the belief space found by the \sbd layer while still allowing exploration.

At each iteration, both \sbd guide and belief search layers are given a user-defined time bound to extend the belief tree. The intuition is to continually feedback feasibility information from the tree search layers to the task planning layer. Weights on $\automaton$ are updated based on \eqref{eq:automatonweights} and the algorithm continues to the next iteration.
\section{Theoretical Guarantees}
\label{sec:analysis}
Here, we analyze the theoretical guarantees of our proposed framework.

\begin{theorem}[Correctness]
    \label{theorem:correctness}
    Let $\text{V}^* = (\text{v}_0, \cdots, \text{v}_f)$ be a returned solution trajectory from the belief search layer. Then, the observation trace of its continuous component is guaranteed to satisfy $\phi$, i.e., $\sigma \models \phi$.
\end{theorem}
\begin{proof}[Proof Sketch]
    $\text{V}^*$ is only returned as a solution if $\text{v}_f.q \in F$. Given polytopic regions $\mathcal{P}$, the computation of chance constraints for labels in \eqref{eq:belief integral} is conservative. Since the belief covariance $(\Sigma$ + $\Lambda)$ is computed exactly, the label is conservatively computed, i.e., $L(b_k)$ is correct. Thus, the belief components of $\text{V}^*$ (belief trajectory $\mathbf{b}$) correctly returns an accepting run on $\automaton$, so $\mathbf{b}$ satisfies $\phi$, i.e., $\mathbf{b} \models \phi$.
\end{proof}






\begin{theorem}[Probabilistic Completeness]
    Given an admissible \sbd, our framework is probabilistically complete.
\end{theorem}

\begin{proof}[Proof Sketch]
    This follows from Thm~\ref{theorem:correctness}, the definition of admissible \sbd, and the results of \cite{Ho2022}. $\automaton$'s edge weights are continually updated according to feasibility estimates, so every accepting run in $\automaton$ will be sampled infinitely often in the limit. In the limit, the search space will be covered by infinitely dense trees for both \sbd guide and belief search layers, from the probabilistic completeness of GBT. From Thm~\ref{theorem:correctness}, any returned solution is guaranteed to be correct.
\end{proof}


Note that we can still achieve probabilistic completeness with an inadmissible \sbd, by setting a separate time bound to the \sbd layer and conducting belief search without a \sbd guide if no solution is found by the \sbd layer. However, admissible \sbd guides can greatly reduce computation times, as seen in our evaluation.
\section{Evaluations}
We demonstrate the effectiveness of our proposed framework, through benchmarks and illustrative case studies. 

\begin{table*}[ht!]
    \centering
    \caption{\small Benchmarking results for underwater inspection scenario. We report the mean time to solution and standard error of the mean. Runs are given a maximum time of $120s$, with successful runs finding a solution within $120s$.
    The best scores are shown in bold fond. 
    }
    \label{tab:benchmarking}
    \scalebox{0.9}{
    \begin{tabular}{l|c|c|c|c|c|c|c|c|c|c}
    \toprule
    \multirow{2}{*}{Algorithm} & \multicolumn{2}{c|}{$\phi_1$} & \multicolumn{2}{c|}{$\phi_2$} & \multicolumn{2}{c|}{$\phi_3$} & \multicolumn{2}{c|}{$\phi_4$} & \multicolumn{2}{c}{$\phi_5$}\\
    & Succ. (\%) & Time (s) & Succ. (\%) & Time (s) & Succ. (\%) & Time (s) & Succ. (\%) & Time (s) & Succ. (\%) & Time (s)\\\hline
    Abs-based& $0$ & NA & $0$ & NA & $0$ & NA & $0$ & NA & $0$ & NA\\
    Simba-free& $\mathbf{100}$ & $7.7 \pm 0.8$ & $66$ & $71.6 \pm 4.3$ & $63$ &  $78.7 \pm 4.0$ & $14$ & $114.1 \pm 1.7$ & $4$ & $117.7 \pm 1.2$\\
    Geo-\sbd& $\mathbf{100}$ & $6.9 \pm 0.6$ &  $89$ & $42.9 \pm 3.6$ & $81$ &  $59.3 \pm 4.2$ & $39$ & $103.5 \pm 2.6$ & $32$ & $107.3 \pm 2.4$\\
    SBA-\sbd& $\mathbf{100}$ & $\mathbf{5.8 \pm 0.5}$ & $\mathbf{95}$ & $\mathbf{29.3 \pm 2.9}$  & $\mathbf{92}$ &  $\mathbf{34.9 \pm 3.2}$ & $\mathbf{73}$ & $\mathbf{85.4 \pm 3.0}$ & $\mathbf{54}$ & $\mathbf{94.5 \pm 3.1}$\\
    \bottomrule
    \end{tabular}
    }
    \vspace{-3mm}
\end{table*}

\textbf{Simulation Benchmarks: }
We compare our proposed framework with (i) a state-abstraction based approach (Abs-based), similar to extending \cite{Maly2013} to the belief space by using a belief space planner, and (ii) using our framework but without a \sbd layer as an ablation study, and using the full framework with different \sbds, (iii) a geometric \sbd that plans directly in the state space (Geo-\sbd), and (iv) SBA-\sbd which includes belief approximation as discussed in Sec.~\ref{sec:simba}.

The scenario is a simplified underwater cave inspection with obstacles as depicted in Figure~\ref{fig:environment1}. Here, the robot can receive measurements with small noise about its state near the surface (in white), while it has to rely on noisier IMU otherwise. The robot has dynamics given by $\dot{x} = v \cos{(\theta)}, \dot{y} = v \sin{(\theta)}, \dot{\theta} = \omega, \dot{v} = a$.
We use state feedback linearization to obtain a linear closed-loop model as in \cite{pairet2021, unicycle}. We consider the following increasingly complex LTLf formulas, with $a = \pi^{0.95}_a$ for region \textsc{a} (in green), $b = \pi^{0.95}_b$ for region \textsc{b} (in yellow), $c = \pi^{0.95}_c$ for region \textsc{c} (in white), and $o = \tilde{\pi}^{0.95}_o$ for region \textsc{o} (brown and black obstacles): 


\begin{itemize}
    \item $\phi_1 = \globally \neg o \wedge \eventually a$
    \item $\phi_2 = \globally \neg o \wedge \eventually(a \wedge \eventually c)$
    \item $\phi_3 = \globally \neg o \wedge \eventually(a \wedge \eventually c) \wedge \eventually(b \wedge \eventually c)$
    \item $\phi_4 = \globally \neg o \wedge \eventually(a \wedge \eventually(c \wedge \eventually(a \wedge \eventually c)$
    \item $\phi_5 = \globally \neg o \wedge \eventually\big(a \wedge \eventually (c \wedge \phi_3)\big) \wedge \eventually\big(b \wedge \eventually (c \wedge \phi_3)\big)$ 
    
\end{itemize}

For each specification, $\globally \neg o$ refers to obstacle avoidance.
$\phi_1$ 
requires to eventually reach region
$a$.
$\phi_2$ is a sequential goal problem, with the goal of reaching region $a$ followed by region $c$.
$\phi_3$ states that the robot should visit regions $a$ and $b$ in any order, before surfacing to region $c$.
$\phi_4$ states that the robot should visit regions $a$ followed by $c$ twice.
$\phi_5$ states that the robot should satisfy $\phi_3$ twice sequentially.

For belief tree planning, we used \cite{Ho2022} with RRT. We provided a time limit of $120s$ to find a solution for 100 trials. The results are summarized in Table~\ref{tab:benchmarking}. From the benchmarking results, it is evident that the state-abstraction approach does not work for belief space problems. This is due to the need to capture measurements affecting the belief dynamics in the abstraction. Note that it may be possible to combine abstraction-based approaches with a biased sampling paradigm, but it is not clear how to do so effectively. We also attempted to use belief discretization instead of state discretization, but it also performed poorly since many transitions in the abstraction are impossible.

Our results show that as complexity of specifications increases, so does the benefit of our planning scheme and the importance of \sbd guides. Even for reach-avoid task $\phi_1$, our planning scheme with both types of \sbds show better performance compared to the other methods, demonstrating the general efficacy of this approach for belief space planning with LTLf specifications. Furthermore, the addition of belief approximation in \sbd allows for more informative guides that reason about the belief space, leading to better solution times for SBA-\sbd as compared to Geo-\sbd. 
Typical solution trajectories returned by our proposed algorithm are shown in Figure~\ref{fig:environment1}, where the robot intelligently navigates close to the surface to localize (reduce uncertainty) before submerging and completing its tasks in the cave, through which measurements have higher noise values.


\begin{figure}[t]
    \centering
    \begin{subfigure}[b]{0.23\textwidth}
            \centering
            \includegraphics[width=\linewidth, , trim={7cm 2.5cm 6cm 2cm},clip]{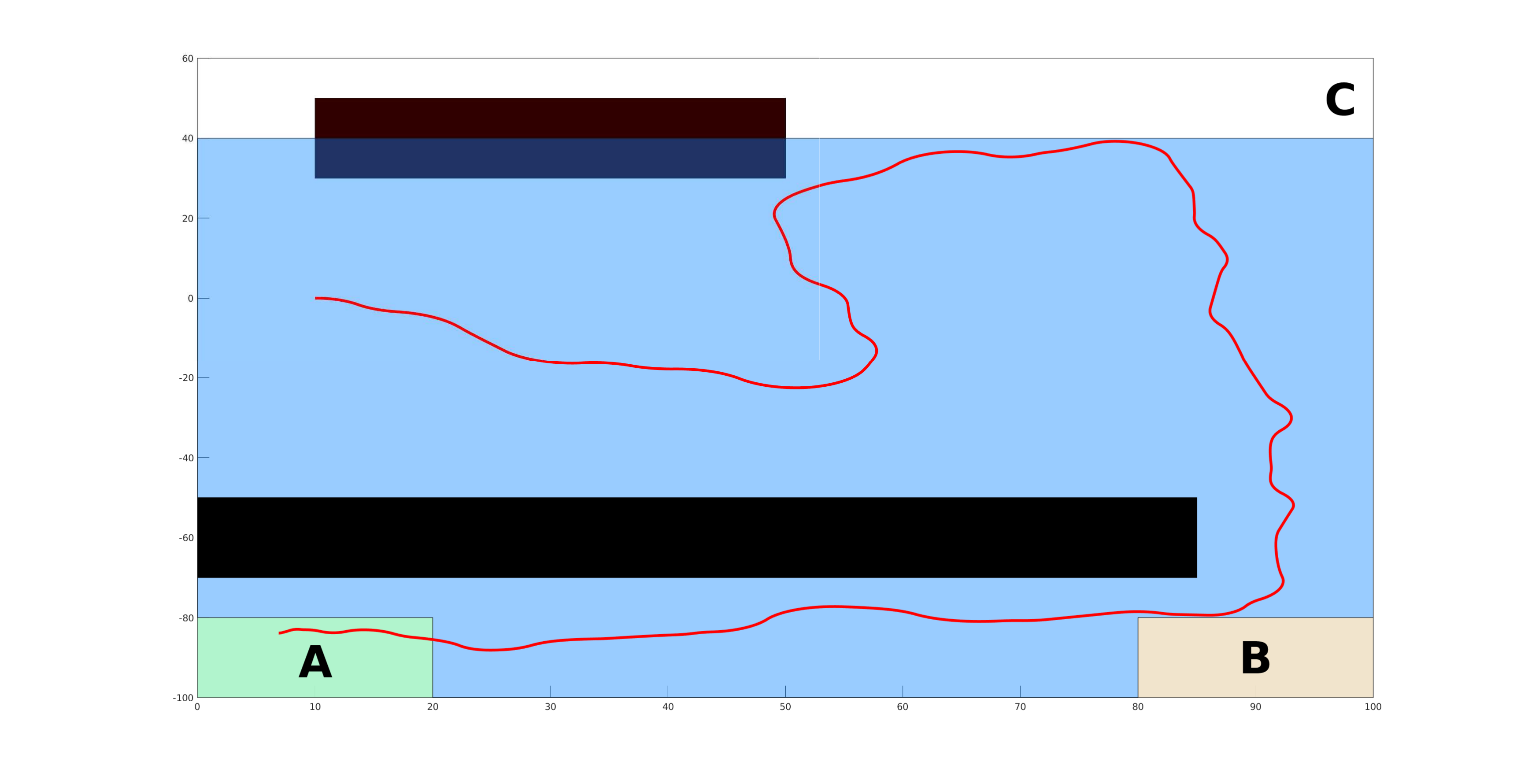}
        \end{subfigure}
        \begin{subfigure}[b]{0.23\textwidth}  
            \centering 
            \includegraphics[width=\linewidth, trim={7cm 2.5cm 6cm 2cm},clip]{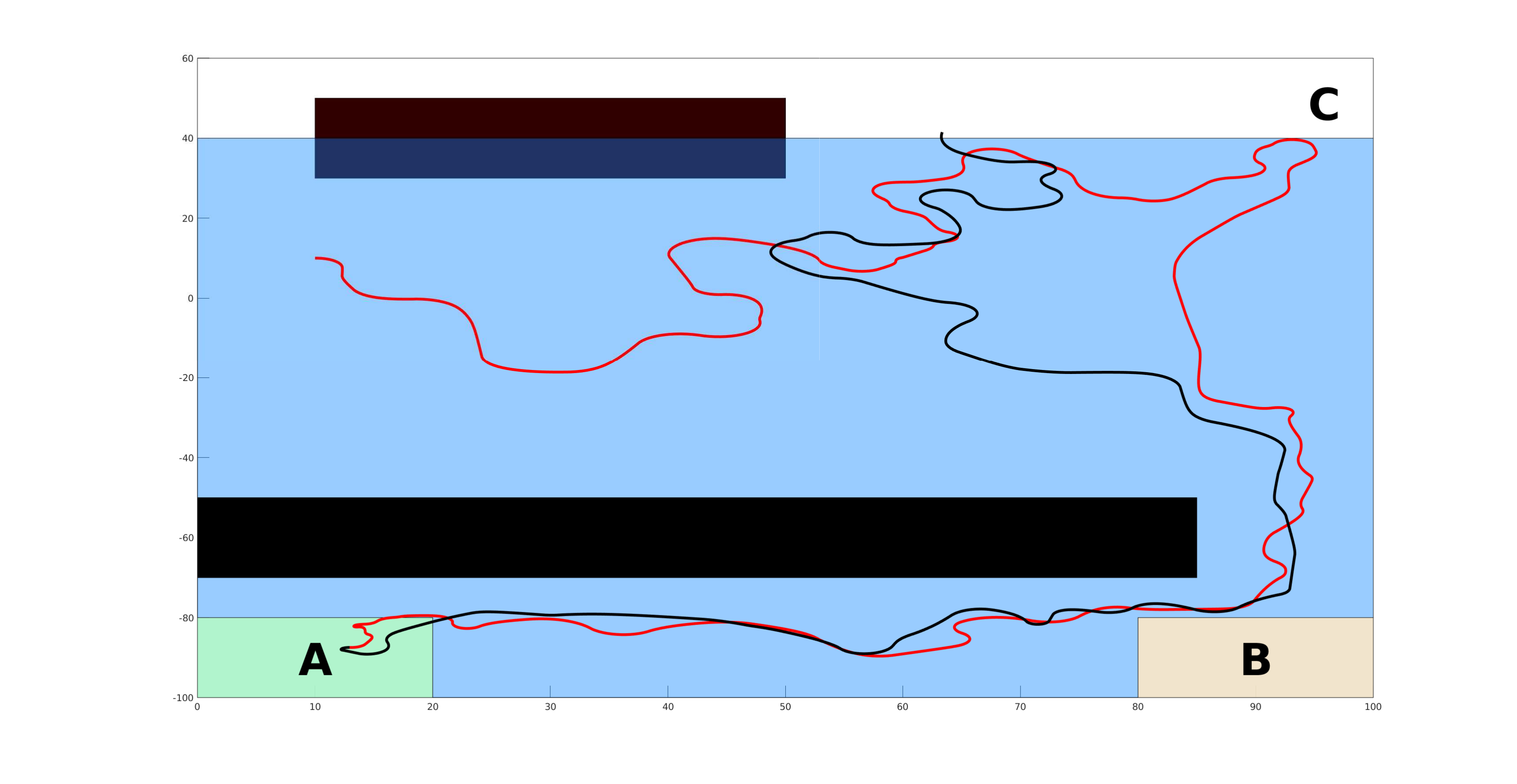}
        \end{subfigure}
        \caption[Effect of varying $p_{bias}$ on solution length with 10 seconds of planning time.]
        {\small Underwater inspection. $\phi_1$ (left): the robot moves close to the surface to localize before completing the task. $\phi_2$ (right): the robot completes sub-task $\eventually a$ (red) and nested sub-task 
        $\eventually c$ (black).
        }  
            \label{fig:environment1}
        \vspace{-3mm}
\end{figure}


\begin{figure}[t]
    \centering
     \begin{subfigure}[b]{0.22\textwidth}  
            \centering 
            \includegraphics[width=\linewidth, trim={7cm 2.5cm 7cm 2cm},clip]{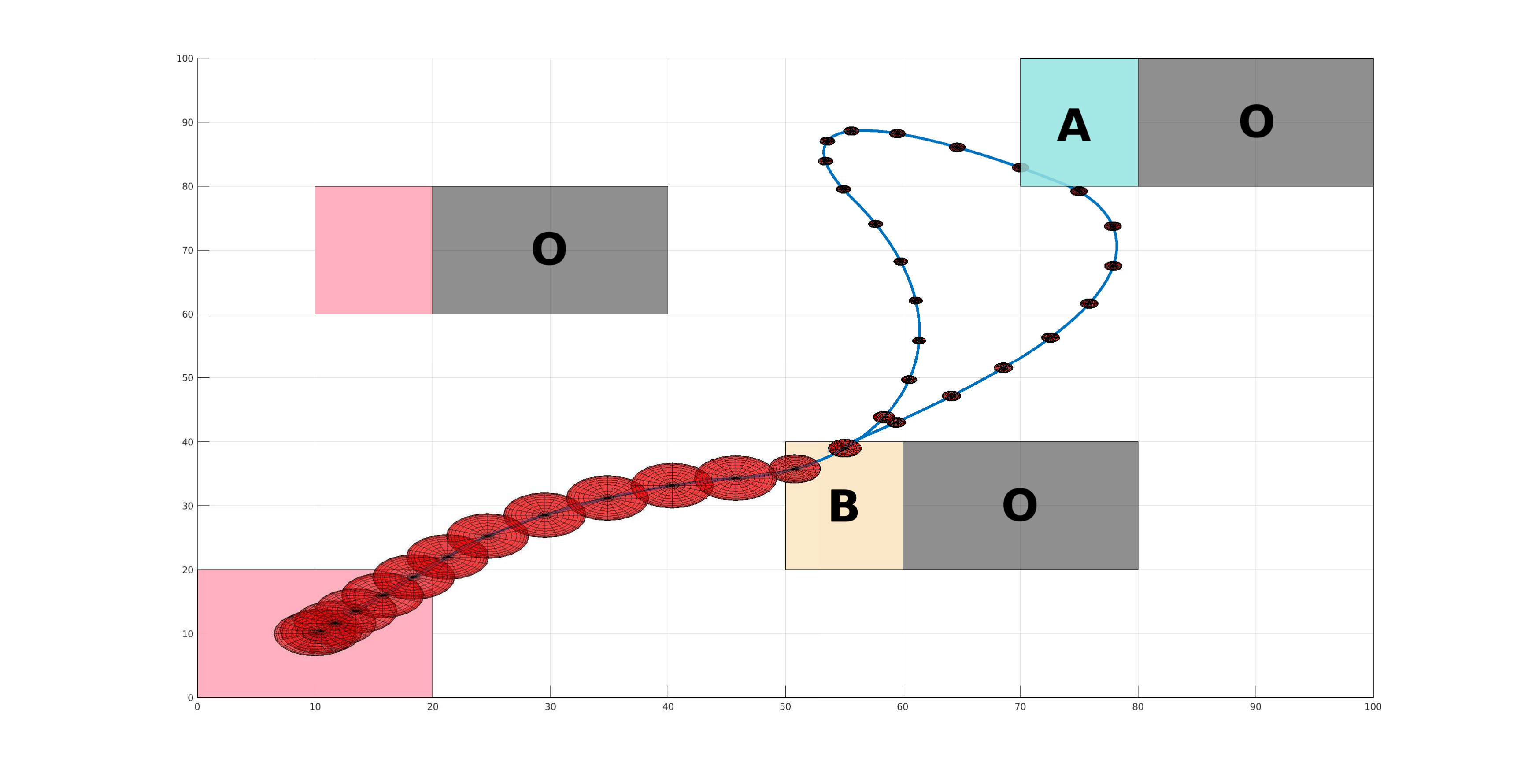}
        \end{subfigure}
        \begin{subfigure}[b]{0.23\textwidth}  
            \centering 
            \includegraphics[width=\linewidth, trim={7cm 2.5cm 7cm 2cm},clip]{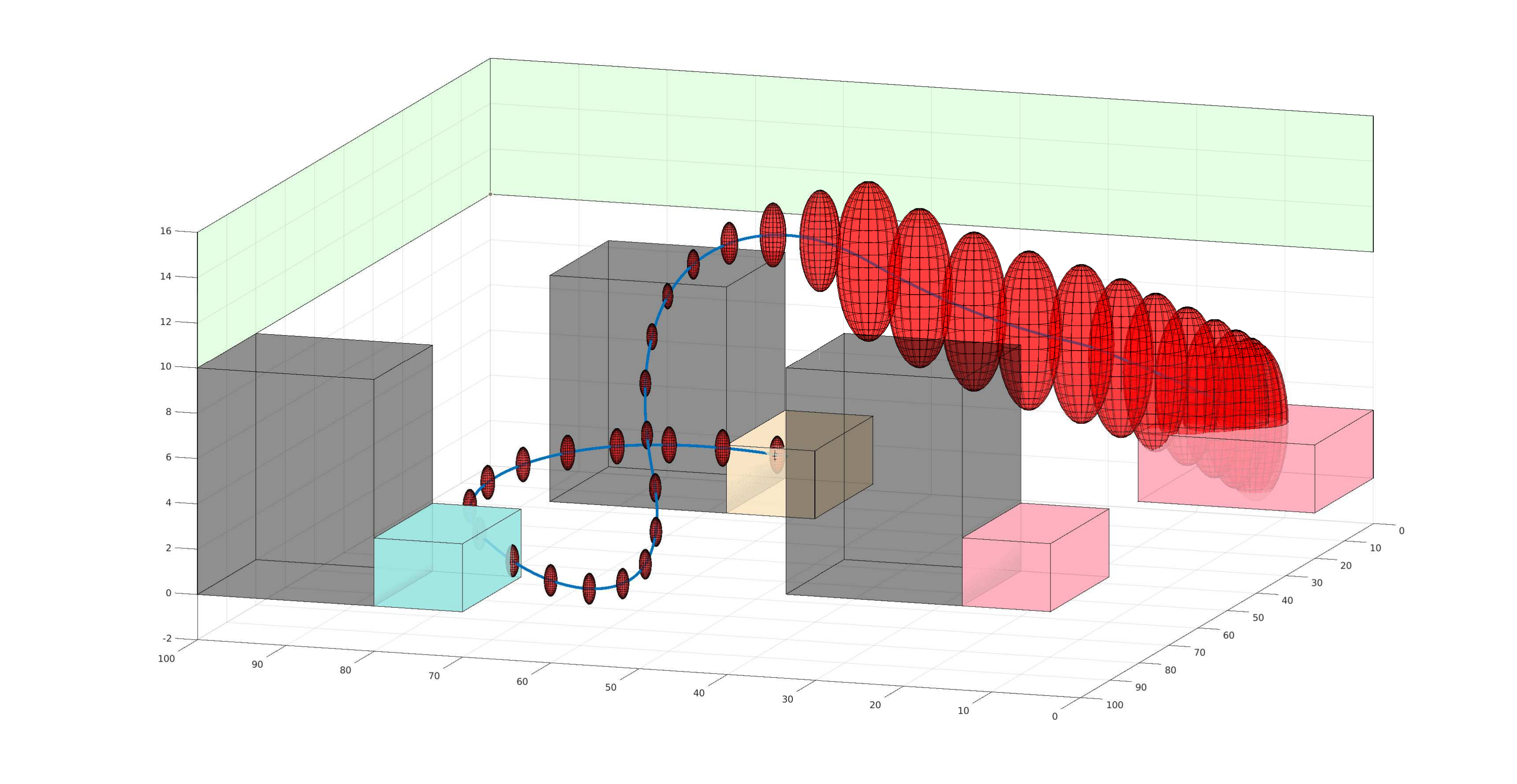}
        \end{subfigure}
        \caption[Environments. ]
        {\small Drone delivery, top down (left) and side (right) views.
        }
        \label{fig:environment2}
        \vspace{-5mm}
\end{figure}

\textbf{Urban Drone Delivery: }
Next, we demonstrate our algorithm in a simulated drone delivery problem in an urban environment, as depicted in Figure~\ref{fig:environment2}.
Details on the dynamics can be found in Appendix~A.2 of \cite{pairet2021}.
We consider a delivery problem, where UAV has to pick up a package at the foot of a building A (region \textsc{a} in blue), and deliver it to foot of building B (region \textsc{b} in yellow) while avoiding obstacles (in gray). Additionally, it can only receive GPS measurements if it flies above the gray buildings (in light green), but it cannot fly that high after picking up a package. In LTLf, this task is $\phi_6 = \globally \neg o \wedge \eventually(a \wedge \eventually b) \wedge \globally (a \implies \globally \neg c)$, with $a = \pi^{0.99}_a$, $b = \pi^{0.99}_b$, $c= \tilde{\pi}^{0.99}_c$, and $o= \tilde{\pi}^{0.99}_{o}$. We conducted $100$ trials for this problem with time limit of $120s$, and our algorithm with SBA-\sbd found a solution with $90\%$ success rate, with mean time to solution of $30 \pm 4.1$. A typical solution is shown Figure~\ref{fig:environment2}. The UAV first flies above the buildings 
to reduce its uncertainty, before flying to region \textsc{a} and then \textsc{b}, to ensure that its uncertainty is low enough to reach both regions with high probability.

\section{Conclusion and Future Work}

We presented a modular framework for motion planning with complex tasks for systems under both motion and measurement uncertainty. We showed that simplified belief models provide important information for guiding motion tree search for planning in belief space. Our proposed framework is correct by construction and probabilistically complete. Empirical evaluations demonstrate the efficiency and efficacy of the planner. Future work aims at analyzing how to design accurate simplified models for better guides.




\bibliographystyle{IEEEtran}
\bibliography{bib}
\end{document}